\documentclass{article}

\usepackage{microtype}
\usepackage{graphicx}
\usepackage{subcaption}
\usepackage{booktabs} 
\usepackage[table]{xcolor}
\usepackage{enumitem}

\usepackage{hyperref}

\usepackage[preprint]{icml2026}

\usepackage{soul}
\usepackage{amsmath}
\usepackage{amssymb}
\usepackage{mathtools}
\usepackage{amsthm} 
\usepackage{thmtools}
\usepackage{thm-restate}
\usepackage{stmaryrd} 

\newtheorem{assumption}{Assumption}

\usepackage[capitalize,noabbrev]{cleveref}

\usepackage{adjustbox}
\usepackage{tcolorbox}

\definecolor{aliceblue}{rgb}{0.94, 0.97, 1.0}
\definecolor{lavender}{rgb}{0.9, 0.9, 0.98}
\definecolor{magnolia}{rgb}{0.97, 0.96, 1.0}
\definecolor{palecornflowerblue}{rgb}{0.67, 0.8, 0.94}

\newcommand{\TEI}{{\rm I}_o}
\newcommand{\AFI}{{{\mathcal{A} \mathcal{F}}_{{\cal I}}}}
\newcommand{\TEISet}{{\cal I}}

\definecolor{dblue}{RGB}{98, 140, 190}
\definecolor{dlblue}{RGB}{216, 235, 255}
\definecolor{dgreen}{RGB}{124, 155, 127}
\definecolor{dpink}{RGB}{207, 166, 208}
\definecolor{dyellow}{RGB}{255, 248, 199}
\definecolor{dgray}{RGB}{46, 49, 49}

\newcommand{\durl}[1]{\textcolor{dblue}{\underline{\url{#1}}}}

\newcommand{\eps}{\varepsilon}

\newcounter{KDefCounter}
\newcommand{\ddef}[2]
{
\vspace{1mm}
\refstepcounter{KDefCounter} 
{\bf Definition \theKDefCounter} (#1): {\it #2}
}

\newtheorem{lemma}{Lemma}

\usepackage[T1]{fontenc} 
\newcommand{\codequotes}[1]{\texttt{#1}}
\newcommand{\tasktext}[1]{\texttt{#1}}
\newcommand{\intenttext}[1]{\text{\emph{#1}}}

\newcommand{\Ifull}{\mathcal{I}}
\newcommand{\Ipartial}{\mathcal{I}_{\mathrm{par}}}
\newcommand{\Nfull}{N^{\mathrm{full}}} 
\newcommand{\Npartial}{N^{\mathrm{par}}}
\newcommand{\Nada}{N^{\mathrm{ada}}} 

\newcommand{\Pfull}{P_{\mathrm{full}}}
\newcommand{\Ppartial}{P_{\mathrm{par}}}
\newcommand{\Pcor}{P_{\mathrm{cor}}}
\newcommand{\Pada}{P_{\mathrm{ada}}}
\newcommand{\indicator}[1]{\left\llbracket{#1}\right\rrbracket}
\newcommand{\epscor}{\varepsilon_{\mathrm{cor}}{}}
\newcommand{\success}{\mathrm{success}}

\DeclareMathOperator*{\argmax}{arg\,max}
\DeclareMathOperator*{\argmin}{arg\,min}

\icmltitlerunning{Affordances Enable Partial World Modeling with LLMs}

\begin{document}

\raggedbottom

\twocolumn[
  \icmltitle{Affordances Enable Partial World Modeling with LLMs}



  \icmlsetsymbol{equal}{*}

  \begin{icmlauthorlist}
    \icmlauthor{Khimya Khetarpal}{yyy}
    \icmlauthor{Gheorghe Comanici}{yyy}
    \icmlauthor{Jonathan Richens}{yyy}
    \icmlauthor{Jeremy Shar}{yyy}
    \icmlauthor{Fei Xia}{yyy,xyz}
    \icmlauthor{Laurent Orseau}{yyy}
    \icmlauthor{Aleksandra Faust}{yyy}
    \icmlauthor{Doina Precup}{yyy}
  \end{icmlauthorlist}

  \icmlaffiliation{yyy}{Google Deepmind}
  \icmlaffiliation{xyz}{Work done at Google Deepmind}

  \icmlcorrespondingauthor{Khimya Khetarpal}{khimya@google.com}

  \icmlkeywords{Machine Learning, ICML}

  \vskip 0.3in
]

\printAffiliationsAndNotice{} 

\begin{abstract}
Full models of the world require complex knowledge of immense detail. While pre-trained large models have been hypothesized to contain similar knowledge due to extensive pre-training on vast amounts of internet scale data, using them directly in a search procedure is inefficient and inaccurate. Conversely, partial models focus on making high quality predictions for a subset of state and actions: those linked through affordances that achieve user intents~\citep{khetarpal2020can}. Can we posit large models as partial world models? We provide a formal answer to this question, proving that agents achieving task-agnostic, language-conditioned intents necessarily possess predictive partial-world models informed by affordances. In the multi-task setting, we introduce distribution-robust affordances and show that partial models can be extracted to significantly improve search efficiency. Empirical evaluations in tabletop robotics tasks demonstrate that our affordance-aware partial models reduce the search branching factor and achieve higher rewards compared to full world models.
\end{abstract}

\newcommand{\fix}{\marginpar{FIX}}
\newcommand{\new}{\marginpar{NEW}}
\newcommand{\greencheckmark}{\color[rgb]{0,.6,0}$\checkmark$}
\newcommand{\redcross}{\color[rgb]{.6,0,0}$\times$}

\section{Introduction}
A key challenge in artificial intelligence (AI) research is to build agents that can generalize across tasks. 
Multi-task reinforcement learning addresses this by leveraging the shared underlying structure across a distribution of tasks. Agents that can generalize across a variety of such sub-tasks have the potential to solve complex, long horizon, tasks. Similar to humans, it is desirable that AI agents are able to represent world knowledge and understand the dynamics of the world. Model-based reinforcement learning (MBRL)~\citep{sutton2018introduction} offers a family of approaches that build a transition dynamics model of the environment (i.e. the MDP's transition function), referred to as a world model, which can then be used for planning. 
While traditional MBRL often relies on interacting with simulators (online) \citep{kakade2003sample} or static datasets (offline), recent work suggests that learning abstract models across tasks can significantly enhance performance \citep{hafner2024masteringdiversedomainsworld}.

With the advent of the large models, whose representations are trained on internet-scale knowledge data, it has been hypothesized that knowledge of the kind provided by a world model, e.g. environment transition dynamics, is already available~\citep{yang2023foundation, bruce2024genie, xi2025rise}. However, work on leveraging such full models in MBRL is in its early stages~\citep{hao2023reasoning, benechehab2024zero, xi2025rise}. Some of the hurdles for applying LLMs to MBRL include the mismatch in observation format between RL agents and LLMs (e.g. in a robotics context, proprioception vs text), the expense in querying LLMs (which can make their use in traditional RL planners, such as Monte Carlo Tree Search prohibitive) and especially the inaccuracies of LLM prediction, including hallucinations (which get exacerbated when used in an RL context, because they can derail the policy optimization).

Conversely, partial models~\cite{talvitie2009simple} focus on making high quality predictions for a subset of state and actions. To exploit this property, we here study the process of leveraging LLMs as a {\em partial} world model. A partial model can guide an RL agent, but its performance as well as inference cost are not a bottleneck, because the LLM is queried only under circumstances where it can provide useful information, which can also be ``debugged" by the RL agent itself. Specifically, we use the LLM to guide an {\em affordance model} for search and planning. Affordances~\citep{gibson1977theory} arise at the interface of the agent and the environment and represent {\em action possibilities}. For AI agents, affordances facilitate adaptive control~\citep{Fikes1972, Korf1983, Drescher1991}, but the agent's understanding of which states afford which actions also allows it to learn partial world models~\citep{talvitie2009simple}, which only apply in limited circumstances and/or only predict relevant parts of the state. For example, an agent might only model temporally extended transitions that complete certain intents~\citep{khetarpal2021temporally}. In this work, we study the effect of inducing such partial world models by leveraging LLMs as both affordances and generative world models. 

Prior work~\citep{abel2014toward, khetarpal2021temporally, khetarpal2020can} has explored affordance-aware partial world models primarily in the context of a single task which is limiting for generalization in the multi-task setting. 
To address this, we extend the formalism of intents and affordances to a multi-task setting which accounts for the agent, the environment, and distribution of tasks. 
See Fig.~\ref{fig:overviewofapproach} where we present a conceptual framework of affordances in the multi-task setting, that considers the distinction between intents internal to an agent embodiment versus the intents specific to tasks presented in conjunction with the environment. 
The delineation between task-specific and task-agnostic intents (i.e. anchored on agent dynamics) in our framework allows for building affordances that consider both the agent embodiment and the physical workspace i.e. the environment resulting in generalized notion of affordances. Similar delineation of task-agnostic and task-specific has been studied by other works~\citep{shin2023pivotalrolelanguagemodeling, flesch2022orthogonal}, where the commonality is that task-agnostic intents are tied to unsupervised reward-free signal~\cite{zhang2020task, burda2018large}, while task-specific intents are closely tied to the context~\citep{sodhani2021multi}.

\begin{figure*}[th!]
    \centering
    \includegraphics[width=1.0\linewidth]{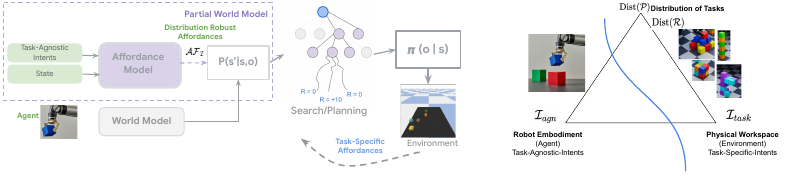}
    \caption{\textbf{Generalizing the concept of affordances to a multi-task setting} is achieved by categorizing agent intents as task-agnostic—grounded in the agent's embodiment, such as a robot's ability to pick or place objects—and task-specific, which depend on the environment and task distribution. By considering the three primary axes of agent, environment, and task distribution, the interplay between these intents facilitates the emergence of affordances at the boundary of the agent-environment interaction. This framework induces partial world models that effectively generalize across tasks, resulting in provably efficient planning.} 
    \label{fig:overviewofapproach}
\end{figure*}

In this paper, we bridge the gap between LLM knowledge and robotic planning by proving that agents capable of achieving task-agnostic, language-conditioned intents necessarily possess predictive partial-world models (Theorem 1). We show that these models, informed by distribution-robust affordances, enable provably efficient search by focusing only on relevant state-action transitions (Theorem 2). We demonstrate how to extract these models in multi-task tabletop robotics by applying a generalized notion of affordances to agent-environment interaction. Our key contributions are:
\begin{enumerate}[itemsep=-0.5cm, parsep=0.5cm]
    \item \textbf{A multi-task affordance framework} that delineates agent-embodied (task-agnostic) intents from environment-driven (task-specific) intents, enabling robust generalization across a distribution of tasks (Sec.~\ref{sec:formalization}). 
    \item \textbf{Theoretical proofs} establishing that (a) achieving a set of language-conditioned intents implies the existence of an underlying partial world model, and (b) using such models significantly improves search efficiency in long-horizon planning (Theorem~\ref{theorem:partialworldmodels}, Theorem~\ref{theorem:fastersearch}; Sec.~\ref{sec:theory}). 
    \item \textbf{Empirical validation in tabletop robotics} demonstrating that leveraging LLMs as both affordance filters and generative world models reduces the search branching factor and achieves higher rewards compared to full world models (Sec.~\ref{sec:experiments}). 
\end{enumerate}

\section{Preliminaries}
\label{sec:background}

\paragraph{Problem Setting.}
We consider table-top robotics tasks $T^1, T^2, \dots T^N$ drawn from a distribution of related  tasks i.e. $T^i \sim \text{Dist}(\mathcal{T})$, where each task $T^i =\langle {\cal S}, {\cal A}, r^i, P^i, \gamma \rangle $ can be viewed as a Markov Decision Process (MDP), comprised of a set of states ${\cal S}$, a set of actions ${\cal A}$, transition dynamics model $P^i(s'|s,a) \sim \text{Dist}(\mathcal{P})$ and reward function $r^i(s) \sim \text{Dist}(\mathcal{R})$, with a fixed distribution of initial states. Robotic tasks typically use fine-grained primitive actions e.g. joint angles of a robot arm, resulting in a complex action space. Therefore, it is common in practice to work with abstract actions (i.e., controllers or options) $o \in {\cal O}$ consisting of predefined low level policies for achieving short term goals, e.g. pick up an \emph{orange block} and place it at location $x, y$. The coordinates associated with block locations depend on the bounds of the physical workspace of the robot gripper. For a task $T^i$, the dynamics models in this paper focuses on  $P^i(s'|s,o)$, that is, the state $s'$ after execution of the abstract action (or option) $o$ is completed starting at state $s \in {\cal S}$. Equivalently we will use $P^i(\tau | s, o))$ for full trajectories (i.e. all state-actions from the start to the end of an abstract action execution).

\textbf{Model-based Reinforcement Learning (MBRL)} usually considers models of the environment that approximate the true model $r^i,P^i$. The focus is typically on models that can produce samples of future states. An approximate model $\hat{M}: \langle\hat{r}, \hat{P}\rangle$\footnote{For simplicity, in this paper we only model transition dynamics $\hat{P}$ and assume the reward function is known through the simulator.} can be leveraged to approximate the optimal policy of the corresponding approximate MDP $\hat{M}$ through a planning or search process. Starting at any state $s$, the model can be used to generate `rollouts` or trajectories modeled with $\langle\hat{r}, \hat{P}\rangle$ and these are used to provide value estimates for action choices $(s,o)$ starting at $s$. The corresponding policy would then select the action with the highest value estimate. Traditionally, world models have been learned from data; during planning, in the setting of a generative model, the model is called repeatedly to generate samples of next states from given states and actions, with the number of calls referred to as the \textit{sample complexity}.

\textbf{Partial Models.} Learning an accurate model can be quite difficult, requiring a lot of data. Moreover, the model does not need to be accurate everywhere, as long as it is accurate in relevant places, and/or it provides useful information for identifying good actions. A useful approach is to build {\em partial models}~\citep{talvitie2009simple}, which only make predictions for specific parts of the state-action space. Partial models come in two flavors: predicting only the outcome of a subset of state-action pairs, or making predictions only about certain parts of the state space.  

\textbf{Affordances.} Building on work from~\citet{khetarpal2021temporally}, we consider partial models of the world that make predictions about the consequences of abstract actions in a given state ${\hat P}(s'|s, o)$ for which we leverage the notions of \emph{intents} and \emph{affordances}. 

\ddef{Temporally Extended Intent $\TEI$ \cite{khetarpal2021temporally}}{ 
A temporally extended intent of abstract action $o\in\mathcal{O}$, $\TEI: \mathcal{S} \to \text{Dist}(\Gamma)$ specifies for each state $s$, a probability distribution over the space of trajectories ${\Gamma}$, describing the \textit{intended} result of executing $o$ in $s$. The associated intent model (i.e. the intended probability of a trajectory $\tau$ when option $o$ is executed from state $s$) is denoted by $P_I(\tau |s,o ) = I_o( s, \tau)$. 
\label{def:intents}
}

\ddef{Degree of intent satisfiability $\zeta$}{A temporally extended intent $\TEI$ is satisfied to degree $\zeta$ 
iff $d(P_I(\tau | s, o), P_{truth}(\tau| s, o))\leq \zeta  , \forall (s, o, \tau) $\footnote{In the context of multi-task setting, $P_{truth}$ is the true distribution associated to a specific task $P_{truth}(  \cdot | T_i)$.}, where $d$ is a user-specified metric between probability distributions and $\tau$ denotes the trajectory starting in state $s$ and following the abstract action $o$. A set of intents $\TEISet$ has satisfiability $\zeta$ iff all its elements have satisfiability $\zeta$.
\label{def:satisiability}}

Definition~\ref{def:intents} describes what an agent expects to happen when it executes an abstract action (e.g. \codequotes{pick-up-orange-block}). Rather than defining a single success state, the intent is a distribution over entire trajectories, allowing variability into how a task might be completed. This view is essential in evaluating whether abstract actions are afforded. An intent is "satisfied" (Definition~\ref{def:satisiability})) only when the realized outcome is sufficiently "close" to the true distribution (i.e. true agent-environment specification). 

\ddef{$\zeta$-affordance \cite{khetarpal2021temporally}}{
Given a set of temporally extended intents $\TEISet=\bigcup_{o \in \mathcal{O}} \TEI$, a $\zeta$-affordance set is a subset $\AFI \subseteq {\cal S} \times \mathcal{O}$, such that $d(P_I(\tau | s,o), P_{truth}(\tau| s, o))\leq \zeta, \forall (s,o)~\in~\AFI$.
\label{def:affordances}
}

An affordance is equivalent to a \emph{selection} of state-action pairs that are likely to achieve corresponding intents~(Definition~\ref{def:intents}). For example, the robot hand is right next to the block; if it runs the \codequotes{pick-up-block} action, the block ends up in the gripper, the intent is met and \codequotes{pick-up-block} is an affordance. The robot hand is across the room; if it runs the \codequotes{pick-up-block} action, it grabs thin air. The intent (block in hand) is not met, hence \codequotes{pick-up-block} is \emph{not} an affordance.

Large Language Models (LLMs) are trained to learn a distribution of the next tokens conditioned on the contexts from the training data. Interestingly, language generation has been well studied as a sequential decision-making problem and language modeling can indeed be understood as an instance of imitation learning~\citep{arora2022exposure}. LLMs have the potential to do well on various reasoning problems~\citep{valmeekam2022large}, albeit using reasoning traces generated in an autoregressive manner. They have also been hypothesized to encode knowledge akin to world models~\citep{hao2023reasoning}, due to extensive pre-training on vast amounts of human data. In the following section, we leverage LLMs as approximate dynamics world models $\hat{P}(s'|s,o)$, to predict future states given start state $s$ and an abstract action $o$. Experimental work will be reported on table top tasks and we assume access to low-level policies for a set of abstract actions. This is departure from prior work~\cite{khetarpal2020can, khetarpal2021temporally}, where the impact of using affordances to learn partial models was studied in toy domains with hand-designed programs that specified intent-completion functions and learned \emph{tabula rasa}. The latter is unfeasible in table-top domains, making it an interesting setting for LLMs as partial models in search and planning.

\section{Generalized Affordance Framework} 
\label{sec:formalization}
We extend the theoretical framework of affordances~\citep{khetarpal2020can} from single-task to multi-task RL. Our approach, illustrated in Fig.~\ref{fig:overviewofapproach}, centers on three primary axes: the \textbf{agent} defined by the robot embodiment e.g. a robot arm gripper, the \textbf{environment} defined by the bounds of the robot's physical workspace e.g. a table top, and the \textbf{distribution of tasks}. 
In this setting, sequential decision making tasks are sampled from this distribution $\mathcal{T}^{i} \sim \mathcal{T}$ comprising a set of tasks related to each other. Once defined, the distribution is used for planning and evaluation. While affordances are fundamentally a function of both the agent and the environment~\citep{chemero2003outline, gibson1977theory}, this dependency becomes uniquely pronounced in multi-task scenarios, as detailed in the following sections.

\textbf{Intents.} Formally, we characterize the set of all intents $\TEISet = \TEISet_{agn} \cup \TEISet_{task}$ in two classes; 1) \textbf{task-agnostic intents} $\TEISet_{agn}$, which are based on an agent's internal dynamics, e.g. a robot embodiment, whereas, 2) \textbf{task-specific intents} $\TEISet_{task}$ captures the set of intents that the environment or the physical workspace poses, depending on the task presented to the agent. 
Task-agnostic intents are the robot's basic skills (what it can physically do), and Task-specific intents are specific desiderata in specific context (what it is trying to achieve right now). Task-agnostic intents can be achieved simply because of how the robot is built (its embodiment) and what is physically in front of it. These intents are general because they do not depend on what the specific task is. When the robot sees a red block, the \intenttext{pick up red block} intent is task-agnostic. It does not matter if the robot's job today is to build a tower, clean the table, or sort by color. As long as the red block is visible and nothing is on top of it, the robot can pick it up. Task-specific intents depend heavily on the high-level task with a certain reward specification. These actions are usually complex combinations that lead to a high reward but are only afforded (possible/valid) in very specific situations. If a the robot is tasked to \tasktext{stack-blocks-while-holding-red-block}, and there is a blue block on the table, \intenttext{stack red block on top of blue block} intent is task-specific. This action is only meaningful if the current task is \tasktext{build-a-stack}. If the task changed to \tasktext{put-all-blocks-in-a-box}, stacking them might be useless or even wrong.

\begin{figure}[h!]
    \centering
    \begin{subfigure}[b]{0.48\linewidth}
    \includegraphics[width=\linewidth]{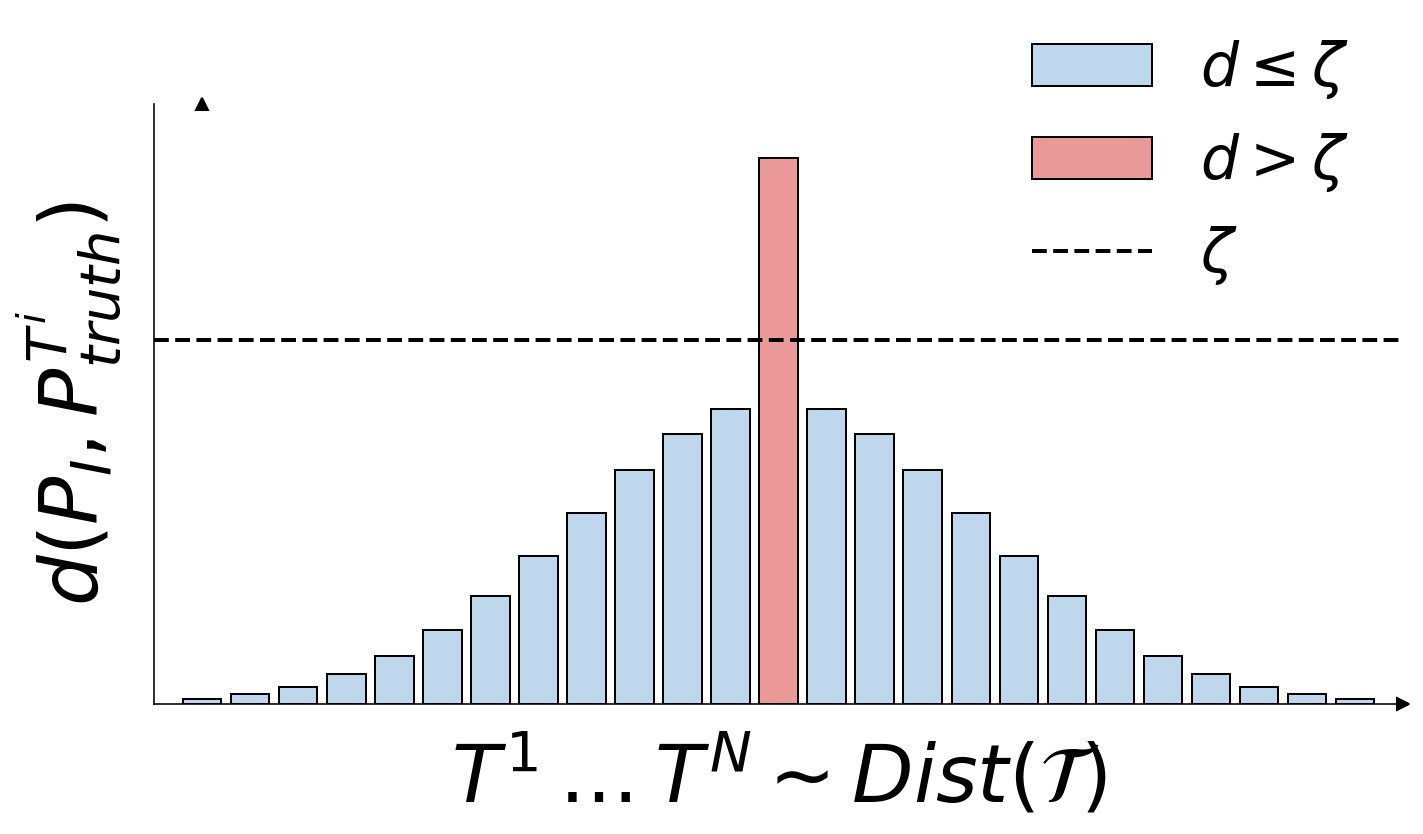} 
    \caption{Task Agnostic Intent} 
    \label{fig:taskagnostic}
    \end{subfigure}
    \begin{subfigure}[b]{0.48\linewidth}
    \includegraphics[width=\linewidth]{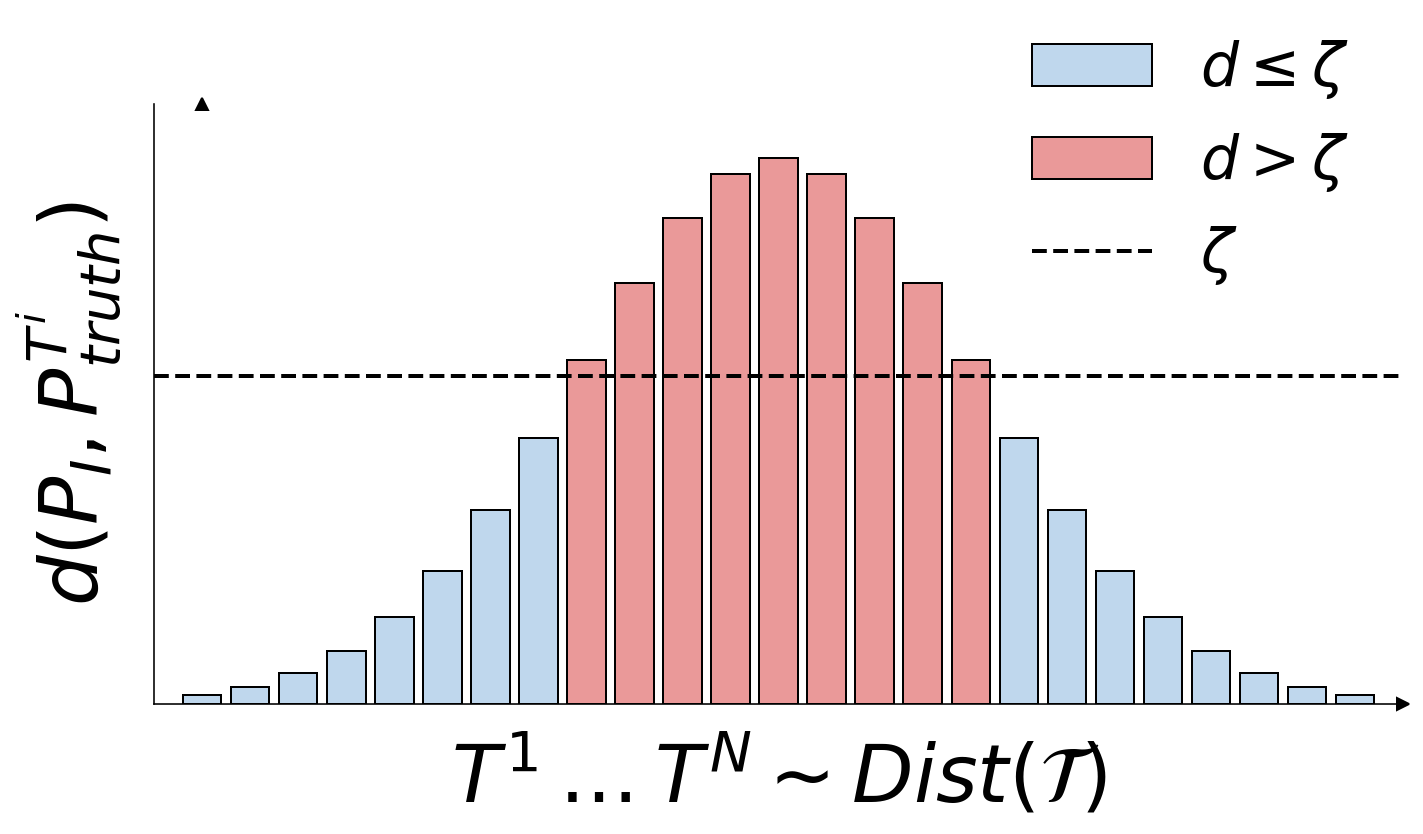}
    \caption{Task Specific Intent} 
    \label{fig:taskspecific}
    \end{subfigure}
     \caption{\textbf{Characterizing intents as task-agnostic and task-specific in the context of multi-task setting}. For an agent embodiment, a task agnostic intent is satisfied to a degree $\zeta$ across the entire task distribution as shown on the left, whereas, a task specific intent doesn't have the same coverage of satisfiability across the task-distribution as its specific to certain tasks only.}
    \label{fig:multitask-intentformluation}
\end{figure}

\textbf{Affordances for Multi-Task Setting.} We now characterize the set of state-options pairs, namely affordances, that facilitate the completion of task-agnostic or task-specific intents. Prior research ~\citep{khetarpal2021temporally} has typically ignored the case where we'd need to work with affordances across multiple tasks selected from a fixed distribution. 

Ideally, affordances would allow the agent to transfer knowledge between tasks with shared underlying structure. We extend here the definition to a probabilistic statement regarding the satisfiability of particular intents, grounded directly in the sampling distribution for the given tasks. This allows us to distinguish between \textbf{distribution-robust agent affordances}, which remain achievable with high probability across diverse task configurations (e.g., grasping an unobstructed block), and \textbf{task-specific affordances}, which may only be satisfied in restricted scenario. For example, in the table-top domain, achieving the intent of grasping a block that has nothing on top should be possible with high probability, consistently across tasks (Fig.~\ref{fig:taskagnostic}), regardless of the block configuration present or desired. However, placing a block inside another one may only be possible in a small number of situations in which there are hollow blocks around of appropriate sizes (Fig.~\ref{fig:taskspecific}). Intuitively, knowledge that is shared across the distribution facilitates transfer to improve learning speed, whereas task-specific knowledge ensures the agent can maximize value for the current objective. Formally,

\ddef{{Distribution-Robust Agent Affordances}}{\label{def:tagnaff}
Given a distribution of tasks $\text{Dist}(\mathcal{T})$, a set of temporally extended intents $\TEISet=\bigcup_{o \in \mathcal{O}} \TEI$, and constants $\zeta, \delta\in (0,1)$, a $\zeta$-affordance set $\AFI$ (see Def.~\ref{def:affordances}) is $\delta$-robust to distribution $\text{Dist}(\mathcal{T})$ iff the probability that the affordance does not satisfy the intent $\TEI$ to a degree $\zeta$ in task $T \sim\text{Dist}(\mathcal{T})$ is bounded by $\delta$. That is, $\forall s, o \in \AFI$ 
\begin{align}
\begin{split}
    P \Big(  d(P_I(\tau | s, o, T), P_{truth}(\tau|s, o, T)) &\geq \zeta  \; \Big| \\ \; T \sim \text{Dist}(\mathcal{T})\Big) \leq \delta \; .
\end{split}
\end{align}
}

To illustrate, we refer the reader to Fig.~\ref{fig:multitask-intentformluation}, which depicts a task agnostic intent that satisfies Def.~\ref{def:tagnaff} grounded in the distribution of tasks, and an intent which doesn't meet the desired criteria resulting in a task-specific intent.

Building on the multi-task affordance formalism, we realize generalized framework depicted in Fig.~\ref{fig:overviewofapproach}, where an agent accesses an imperfect abstract world dynamics model $P(s'|s,o)$. Unlike naive full world models that inefficiently predict over the entire state-action space—often leading to inaccurate Monte-Carlo planning results (Sec.~\ref{exp:3blockstask})—we induce a \textit{partial world model} to restrict predictions to a relevant subset. We input the current state and a set of $k$ pre-defined intents $\TEISet$ into an affordance model to yield distribution-robust affordances $\AFI$ (Def~\ref{def:tagnaff}). These affordances constrain the world model's predictions, substantially reducing the search branching factor and improving action-value estimates for the policy $\pi$. While the distribution-robust affordances are relevant across tasks, the agent is implicitly informed of task-specific intents via the reward signal during the subsequent environment interaction.
\begin{figure}[h!]
    \centering
    \includegraphics[width=0.7\linewidth]{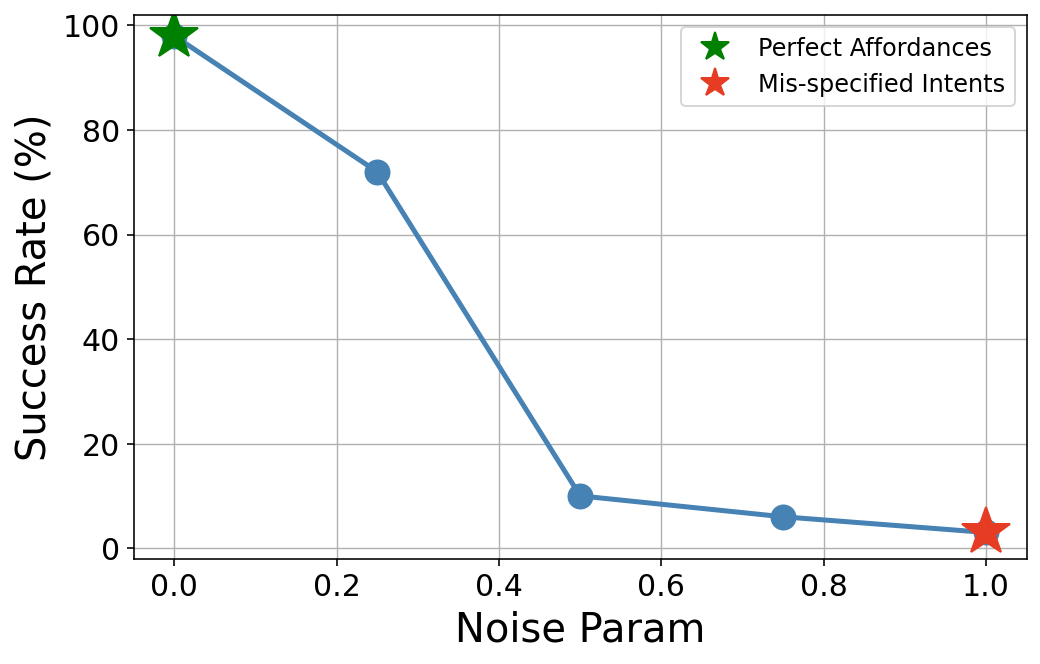}
    \caption{\textbf{Improvements in affordance model accuracy results in better performance of the search policy.} Given a fixed budget for search, a perfect world model, any errors in affordance prediction translates to catastrophic failures in the search policy.}
    \label{fig:demo_progamaticaffordances}
\end{figure}

\textbf{Illustration.} To demonstrate the utility of the generalized affordance framework, we evaluate the pipeline shown in Fig.~\ref{fig:overviewofapproach} on a 3-block stacking task using a simulated robot arm. By employing a perfect world model (the simulator) alongside programmatic affordances derived from task-agnostic intents, we examine how varying the accuracy of affordable actions affects search efficacy. As shown in Fig.~\ref{fig:demo_progamaticaffordances}, higher affordance model accuracy directly correlates with improved search policy performance (green
star). Conversely, mis-specified affordances lead to catastrophic failures regardless of the world model's fidelity, particularly under constrained search budgets (red star). This highlights that since affordances emerge at the boundary of agent-environment interaction, grounding the agent’s affordance beliefs in an adaptive fashion is critical for robust planning.

\section{Theoretical Analysis}
\label{sec:theory}
We now provide a theoretical analysis of the multi-task affordance framework. In Sec.~\ref{sec:partialmodelencodedresult}, we establish that achieving a set of intents implies the existence of an underlying partial world model through Theorem~\ref{theorem:partialworldmodels}. In Sec.~\ref{ref:speedupresult}, we study the efficiency in search enabled by partial world models. We show that at a small price, adapting the model enjoys the potential exponential speedup of the partial model, while being robust to missing intents from the full model.

\subsection{Distribution-Robust Agent Intents induce Partial World Models}
\label{sec:partialmodelencodedresult}
To facilitate the theoretical analysis of our generalized framework, we specialize the class of sequential tasks under consideration. 
We follow the standard formalism for expressing sequential tasks using linear temporal logic, with a focus on three simple types of tasks: i) \textit{actions}, taking an (abstract) action $O = o$, ii) \textit{outcomes}, achieving a desired outcome from an action, denoted $X(s\in g)$ where $g\subseteq \mathcal S$, iii) \textit{goals}, reaching a goal state $F(s\in g)$. 
Here $X(s\in g)$ requires the desired state(s) $s$ to be reached as an immediate outcome of the previous (abstract) action $O = o$, and $F(s\in g)$ requires only that the goal state is reached at some point in the future.  
For example, stack blocks in r / g / b order, $F(S \in g_{rgb})$. 
To denote tasks comprising of sequences of these primitives we use angular brackets e.g. $T^i = \langle o, X(s), \ldots \rangle$. 
See \cref{app:prooftheorem1} for details. 
For the purposes of our analysis, we define depth-n tasks~\citep{nangue2020boolean}.

\ddef{{Depth-n tasks}}{ \label{def: depth n tasks}
Given a set of abstract actions $\mathcal{O}$ and states $\mathcal S$, we consider the set of depth-$n$ tasks $\mathcal T(n)$ to be the set of tasks $T^i$ that can be solved by achieving a sequence of at most $n$ \textit{actions}, \textit{outcomes}, or \textit{goals} e.g. $T^i = \langle o_{i,0}, \ldots, o_{i,m}\rangle$ where $m\leq n$.}

\begin{assumption} 
\label{ass:communicatingCMDP}
The MDP with transition function $P(s'\mid o, s)$ is communicating (i.e. there exists a policy that gets you from any one state to any other state in a finite number of steps with non-zero probability)~\citep{bartlett2012regal}, and there exists a state $s$ such that at least two actions are afforded to the agent.  
\end{assumption}

To ground this formulation in planning, an agent is said to be a competent planner when it can achieve high depth-n tasks with a high probability. Formally, the optimal planner for depth-n tasks satisfies,
$$
    \pi^*(o \mid s, T) = \argmax\limits_{\pi} P_\pi(\tau \models T\mid s), \forall \, s\in \mathcal S, T \in \mathcal{T}(n),
$$
where $P_{\pi}(\tau \models T\mid s)$ is the probability that the trajectory generated by $\pi$ from initial state $s$ corresponds to achieving the task $T$. We can then define an agent's distribution-robust competence for tasks in line with Def.~\ref{def:tagnaff}, by choosing the distance metric $d(\TEI(s, \tau \mid T), P(\tau \mid \pi,s, T))$ to be the difference in probability mass over trajectories that satisfy the task $T^i\in \mathcal T$ for the agent, and the corresponding probability mass for the optimal agent with the same affordance set.
We also choose $\text{Dist}(\mathcal T)$ to be an arbitrary task distribution that includes $\mathcal T(n)$, and assume there is some upper bound $\delta$ on the probability that the agent violates the $\zeta$-regret bound for any given task $T^i \sim \mathcal T(n)$ independent of the agent's performance on any other task $T^j \sim \mathcal T(n)$.

\ddef{$(n, \zeta, \delta)$-optimal agent}{
\label{def:agent}
Given affordances $\AFI$ (Def.~\ref{def:tagnaff}), an  $(n, \zeta, \delta)$-optimal agent is such that
{
\[ P ( \mathbf{R}(\pi_i, T^i) \geq \zeta \, \mid \;  T^i \sim \mathcal T (n)) \leq \delta. \]
}
where, $\pi_i$ is the policy of the agent for a depth-$n$ task $T^i \sim \mathcal T(n)$ and 
{
\[
    \mathbf{R} (\pi, T^i) :=\max\limits_{\pi'} \left\{ P_{\pi'}(\tau \models T^i)\right\} - P_{\pi}(\tau \models T^i).
\]
}}
We can interpret $\mathbf{R}(\pi, T^i)$ as the agent's expected regret for task $T^i$, where the reward function is $+1$ if the agent achieves the task and $0$ otherwise. 
Therefore, a $(n, \zeta, \delta)$-optimal agent describes an agent that is a competent planner to depth $n$, where the probability the agent violates a $\zeta$-regret bound for any depth-n task is bounded by some $\delta$. 

\begin{restatable}[Distribution-Robust Affordances induce Partial World Models]{retheorem}{restatemaintheorem}
\label{theorem:partialworldmodels}
Let $\pi$ be the policy of a deterministic $(n,\zeta, \delta)$-optimal agent (Def.~\ref{def:agent}), in an environment satisfying Assumption~\ref{ass:communicatingCMDP}. 
$\pi$ encodes a partial world model $\hat{P}_{\mathrm{par}}(s'\mid o, s)$, and the worst case error for this world model is bounded by, 
{
\[
    P ( \big| \hat{P}_{\mathrm{par}}(s'\mid o, s) -  P(s'\mid o, s)| \big| \geq  \phi + \epsilon ) \leq \frac{\rho^{n\, \epsilon}}{1 - \rho} \; ,
\]
}
$\forall (s,o) \in \AFI $, where
{
\[
    \phi = \frac{1}{2}\sqrt{\frac{(1 + \zeta)}{ n\, (1- \zeta)}} \, ,\quad \rho = 2\sqrt{\delta (1-\delta)}\, \text{, and }  \epsilon > 0 .
\]
}
\end{restatable}

Proof is in Appendix~\ref{app:prooftheorem1}. Theorem \ref{theorem:partialworldmodels} implies that there is a characteristic scale for the error in the partial world model we can recover from the agent's policy, given by $\phi$, which scales with $\mathcal O(n^{-1/2})$, but increases rapidly for higher intent mis-specification i.e. $\zeta$ close to 1 (where the regret bound becomes very weak). For $\delta < 1/2$ the probability that the error is $\epsilon$ larger than the characteristic error $\phi$ decays exponentially in $\epsilon$, with an exponent given by $n$. From this we conclude that for agents that perform well on a task distribution involving planning over its afforded actions to high depth (n), the agent's policy encodes an accurate partial world model with respect to those afforded actions. Specifically, we can recover a partial world model from the agent's policy that deviates from the true transition probabilities with an error that decays exponentially. 

\subsection{Partial World Models Enable Provably Efficient Search}
\label{ref:speedupresult}

We are interested in understanding the increase in search efficiency when using a partial world model $\Ppartial$ as opposed to a full model $\Pfull$. When sampling from $\Pfull$, the agent considers all $n = |\mathcal{O}|$ possible actions, which translates to $n = |\Ifull|$ intents. Alternatively, sampling from $\Ppartial$ is based on a predefined set of $k$ task-agnostic and -specific intents from $\TEISet$, all independently achievable and satisfied to degree $\zeta$ (Def.~\ref{def:intents}) across a distribution of tasks (Def.~\ref{def:tagnaff}). In other words, $\Ppartial$ has access to $k=|\Ipartial|$ intents ($\Ipartial\subseteq\Ifull$). 

\begin{assumption}
\label{ass:finitesmallsetofintents}
The size of the set of task-agnostic agent intents $\Ipartial$ is finite and much smaller than the number of all possible abstract actions $n = |\mathcal{O}|$ to be considered at any state in the search process: $k \ll n$.
\end{assumption}

The appeal of a partial model is due to the potential gigantic exponential speedup that can be obtained by greatly reducing the branching factor of complex combinatorial environments: suppose a task can only be solved in the full model by a single specific sequence of $L$ abstract actions, thus having an expected time to success of $n^L$. If this task can also be solved with the $k$ intents of the partial model, the latter would take $k^L$ steps to success in expectation. Using Assumption~\ref{ass:finitesmallsetofintents}, the speed up from the partial model can be so large as to make the intractable task tractable.

However, the pre-specified intents might not be perfectly aligned, or worse, if the partial model misses \emph{any necessary} intent to solve \emph{any} task, the expected number of trials to solve every task by the partial model $\Npartial$ is infinite. To protect against this, we define a \textbf{\emph{corrected} partial model} $\Pcor$ which assigns a small probability $\epscor$ to account for missing intents from the full model: for each intent $I_j$,
{\begin{multline*}
    \Pcor(I_j\mid \epscor) \\
    = (1-\epscor) \frac1k\indicator{I_j \in \Ipartial} + \epscor\frac1{n-k}\indicator{I_j \notin \Ipartial}\,.
\end{multline*}
}
Note that setting $\epscor=0$ recovers the partial model, while $\epscor=(n-k)/n$ recovers the full model. Moreover, we develop an \textbf{\emph{adaptive} model} $\Pada$ (see \cref{app:prooftheorem2}) that automatically competes with the optimal $\epscor$ \emph{for each task}, at the cost of a small $O(L)$ factor:

\begin{restatable}[Provably Efficient Search with a Partial Model]{retheorem}{restateefficiencytheorem}
\label{theorem:fastersearch}
The expected number of samples necessary for the adaptive corrected partial model to solve every task is at most a factor $O(L)$ of the best corrected model for each task:
{
\[
    \Nada \leq e(2L+3) \sum_{i\leq |\mathcal{T}|} \frac1{\max_{\epscor\in[0,1]}\Pcor(\success_i|\epscor)}
\]
}
where $L$ is the length of the sampled sequences of intents.
\end{restatable}

Proof is in~\cref{app:prooftheorem2}. The intuition here is that at the price of a small factor $(2L+3)e$, the adaptive corrected model enjoys the potential exponential speedup of the partial model, while being robust to missing intents from the model. 

\begin{table*}[t!]
\caption{\textbf{Inducing a partial model via affordances is more efficient in search, albeit at the cost of induced bias due to affordances in the block arrangement task.} Partial models are able to get to the reward in fewer search simulations as compared to a full world model, which does not get to any reward in limited budget of $4$ simulations. Results are averaged across $4$ independent seeds. Knowledge of affordances as shown here with the oracle baseline results in both, reduced number of MC simulations, and much fewer LLM calls.}
\begin{adjustbox}{width=2.1\columnwidth,center}
\begin{tabular}{c|c|c|c|c|c}
\hline
\hline
\textbf{[3 blocks]}
\cellcolor{aliceblue} \textbf{Method}  & \cellcolor{aliceblue} \textbf{Model}  & \cellcolor{aliceblue}  \textbf{MC-Search Score}  & \cellcolor{aliceblue} \textbf{Simulations} &
\cellcolor{aliceblue}  \textbf{LLM Calls}  &
\cellcolor{aliceblue}  \textbf{Steps to Completion} \\
\hline
\textcolor{red}{Full Model} w/ Few Shot Examples &  $P_{\texttt{LLM}}(s’|s,o)$ & \textcolor{red}{$0.0$} & \textcolor{red}{$>4$} & $40.0$ & \textcolor{red}{$>10$} \\
\hline
\textbf{Partial Model (Ours)} & $P_{\texttt{AFF-LLM}}(s’|s,o)$ & \cellcolor{aliceblue}  \textbf{$3.75$} & \cellcolor{aliceblue} \textbf{$2.75$} & $43.25$ & \cellcolor{aliceblue} \textbf{$5$} \\
\hline
Partial Model w/ \textcolor{cyan}{Oracle} Affordances & $P_{\texttt{Oracle-AFF-LLM}}(s’|s,o)$ & $5.0$ & $2.5$ & $15.75$ & $2$ \\
\hline
\hline
\textbf{[5 blocks]}
\cellcolor{aliceblue} \textbf{Method}  & \cellcolor{aliceblue} \textbf{Model}  & \cellcolor{aliceblue}  \textbf{MC-Search Score}  & \cellcolor{aliceblue} \textbf{Simulations} &
\cellcolor{aliceblue}  \textbf{LLM Calls}  &
\cellcolor{aliceblue}  \textbf{Steps to Completion} \\
\hline
\textcolor{red}{Full Model} w/ Few Shot Examples &  $P_{\texttt{LLM}}(s’|s,o)$ & \textcolor{red}{$0.087$} & \textcolor{red}{$> 4$} & $40$ & \textcolor{red}{$> 10$} \\
\hline
\textbf{Partial Model (Ours)} & $P_{\texttt{AFF-LLM}}(s’|s,o)$ & \cellcolor{aliceblue}  \textbf{$0.2239$} & \cellcolor{aliceblue} \textbf{$2.215$} & $40$ & \cellcolor{aliceblue} \textbf{$5$} \\
\end{tabular}
\label{tab:3blocksenv_mcts}
\end{adjustbox}
\end{table*}

\section{Empirical Illustration: LLM as Partial World Model for Efficient Planning}
\label{sec:experiments}
To illustrate the concepts empirically, we consider an agent with access to a world dynamics model $P(s'|s,o)$ represented by a pre-trained text-only large model.\footnote{We note that our approach is generally applicable to any large model e.g. VLM's can also be used.} Our approach concretizes the general framework depicted in Fig.~\ref{fig:overviewofapproach} by leveraging LLMs both for affordances and for world modeling (See Appendix Fig.~\ref{appfig:approach_overview}). To bypass the need for manual intent-completion programs and expensive data collection, we leverage an LLM as an affordance model $\texttt{f}_\texttt{affordances}(s_t, m)$. Given a natural language state representation $s_t$, the LLM generates $m$ affordable actions. These correspond to distribution-robust affordances and induce a partial world model $\Ppartial$, denoted as $P_{Aff-LLM}(s'| s, o)$ in this section. 

We empirically study; \textbf{[Q1.]} Can LLMs as affordance-aware partial world models lead to improved efficiency in search? (Sec.\ref{exp:3blockstask}), and \textbf{[Q2.]} Can we leverage distribution-robust agent affordances to induce useful partial world models across related tasks? (Sec.\ref{exp:5-7-blockstask}) We consider a simulated tabletop for long-horizon manipulation tasks using the pybullet simulator. We investigate block-rearrangement tasks with 3, 5, and 7 blocks to test the agent's performance as difficulty increases. For each configuration, the agent is tasked with satisfying the high-level task description \tasktext{move\_blocks\_close\_to\_each\_other}.

\subsection{LMs as partial world models can substantially improve the sample search efficiency}
\label{exp:3blockstask}

Given only the high-level language instruction, the task requires multi-step reasoning and planning. We first consider the 3-block configuration. We compare three methods, namely 
\begin{enumerate}
    \item \textbf{Full World Model} w/ Few Shot Examples: a pre-trained LLM as a generative world model $P_{\texttt{LLM}}(s'|s,o)$ prompted with the state $s$, action $a$, and few-shot (4 shots) examples to predict the future state $s'$, 
    \item \textbf{Partial Model (Ours)}: is an induced partial model, where given a state $s$, first a pre-trained LLM representing $\texttt{f}_\texttt{affordances}(s_t, m)$ generates m affordable actions, inducing a partial world model $P_{\texttt{AFF-LLM}}(s'|s,o)$ prompted with only state $s, o \in \AFI$, and few-shot (4 shots) examples to predict the future state $s'$, and
    \item \textbf{Partial Model with oracle affordances}: considers our method with programmatically  generated ground truth affordances to illustrate the bias induced from pre-specified intents. 
\end{enumerate}
Based on robot embodiment of a gripper arm \textbf{task agnostic intents} $\mathcal{I}_{\texttt{agn}}$ are specified as natural language description, with the idea of capturing invariance across tasks, resulting in $\AFI$. For e.g. ``a block is movable only if that color is present in the current state and when nothing is on top of it.'' We show more examples of these prompts in Appendix~\ref{app:task-agnostic-intents}. We use a simple implementation of Monte-Carlo (MC) sampling to build a policy from a world model (Fig.~\ref{fig:mcts}). We use the true reward model on each node (i.e. tasks are defined on states). From a real environment state, we expand a search tree using a model. We run $~4$ simulations, and each simulation expands a new un-visited node (using a random policy over untried actions). Random policy rollouts extend trajectories to at most 10 states (depth), seeking states with reward +10. We assess the likelihood that the model finds rewarding states (i.e. \textit{MC average reward}) averaged over multiple independent seeds. Whether modelled successful trajectories match trajectories in real test environment. (i.e. \textit{online policy evaluation}). When using a partial model, the MC tree is only expanded on affordable actions. In doing so, the search tree focuses modelling resources on trajectories that match predefined task-agnostic intents. 

\begin{figure}[h!]
    \centering
    \includegraphics[width=0.7\linewidth]{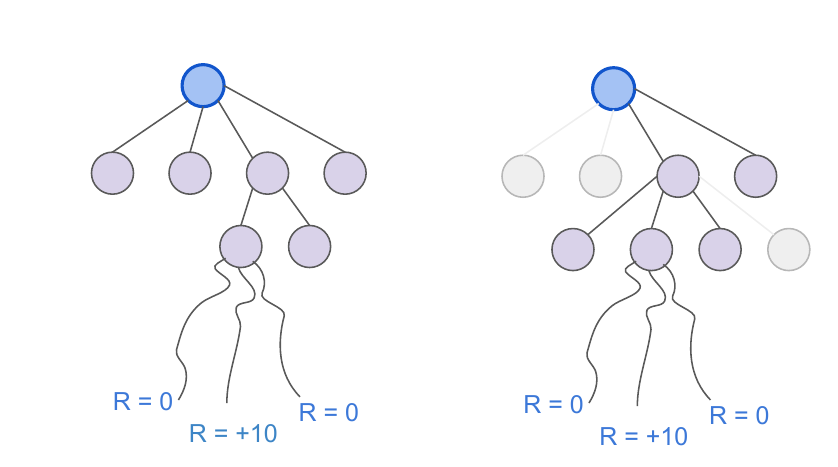}
    \caption{\textbf{Monte-Carlo sampling with a partial model} considers expanding the tree is only on affordances. However, search with a full model would expand all possible state-action pairs.} 
    \label{fig:mcts}
\end{figure}

\begin{table*}[ht!]
\caption{\textbf{Distribution-robust agent affordances are robust to related tasks.} We report the \textbf{1. avg leaf-node Score}, \textbf{2. avg tree score} reflecting the quality of the search tree on average across all nodes, \textbf{3. avg accumulated reward} reflecting the average accumulated reward by running MC-search as open loop policy on states that the model imagined, and \textbf{4. Online Policy Evaluation} reflects cumulative reward during online policy evaluation for search policy. Search guided with affordance-aware partial models constructs an overall better quality of tree, and results in a better policy for online RL as compared to the full model. Results are averaged across $4$ independent seeds with standard deviation in the brackets.}

\label{tab:5-7-blocksenv_mcts}
\centering
\begin{adjustbox}{width=2.1\columnwidth,center}
\begin{tabular}{c|c|c|c|c|c}
\hline
\hline
\cellcolor{aliceblue} \textbf{Method}  & \cellcolor{aliceblue} \textbf{Model}  & \cellcolor{aliceblue}  \textbf{Avg Leaf-Node Score}  &
\cellcolor{aliceblue}  \textbf{Avg Tree Score}  &
\cellcolor{aliceblue}  \textbf{Avg Accumulated Reward} &  \cellcolor{aliceblue} \textbf{Online Policy Evaluation} \\
\hline \cellcolor{aliceblue} \textbf{\texttt{num blocks} = 5} \\ \hline
\textcolor{red}{Full Model} with few-shot examples. &  $P_{\texttt{LLM}}(s'|s,o)$ & $0.089 (\pm 0.04)$ & $0.089 (\pm 0.04)$ & $0.65(\pm 1.26)$ & $0.7895(\pm 0.25)$ \\
\hline
\textbf{Partial Model (Ours)} & $P_{\texttt{AFF-LLM}}(s'|s,o)$ & \cellcolor{aliceblue}  $0.235 (\pm 0.087)$ & \cellcolor{aliceblue} $0.2314 (\pm 0.086)$ & \cellcolor{aliceblue} $0.9212(\pm 0.709)$ & \cellcolor{aliceblue} $1.56 (\pm 0.45)$ \\
\hline \cellcolor{aliceblue}  \textbf{\texttt{num blocks} = 7} \\ 
\hline
\textcolor{red}{Full Model} with few-shot examples. &  $P_{\texttt{LLM}}(s'|s,o)$ & $-0.01134 (\pm 0.057)$ & $-0.01134 (\pm 0.057)$ & $ -0.1869 (\pm 0.51)$ & $-0.836 (\pm 0.1059)$ \\
\hline
\textbf{Partial Model (Ours)} & $P_{\texttt{AFF-LLM}}(s'|s,o)$ & \cellcolor{aliceblue}  $0.0180 (\pm0.12)$ & \cellcolor{aliceblue} $0.020 (\pm 0.12)$ & \cellcolor{aliceblue} $0.386 (\pm 0.57)$ & \cellcolor{aliceblue} \textbf{$ 0.2141 (\pm 0.617)$} \\
\hline
\end{tabular}
\end{adjustbox}
\end{table*}

For the 3-blocks tabletop task configuration, we observe in Table~\ref{tab:3blocksenv_mcts} that inducing a partial model via affordances $P_{\texttt{AFF-LLM}}$ is more efficient in search, albeit at the cost of induced bias due to affordances. Not knowing the precise affordances requires marginally more LLM calls in the than using the perfect knowledge of affordances i.e. $P_{\texttt{Oracle-AFF-LLM}}$. The number of MC simulations it takes with each method is lower  when using a partial model (2.75) as compared to a full model $P_{\texttt{LLM}}$ (which exhausts the budget of \textcolor{red}{$> 4$} simulations and never reaches a rewarding state with MC Score of \textcolor{red}{$0.0$}). Besides, our approach also results in a better performing MC policy assessed with online evaluation in the test environment where it takes 5 steps to task completion on an average as compared to using a full world model (which exhausts the steps to task completion budget of \textcolor{red}{$> 10$} steps and does not complete the task across all runs).

\begin{tcolorbox}[colback=gray!10, colframe=gray!10, arc=4pt, boxrule=0pt]
\textbf{Key Insight I:} Inducing a partial world model via affordances is more efficient in search, albeit at the cost of induced bias due to affordances (See Table~\ref{tab:3blocksenv_mcts}).
\end{tcolorbox}

\subsection{Distribution-robust agent affordances facilitate partial models that improve planning across tasks}
\label{exp:5-7-blockstask}
Next, we ask if we can leverage distribution robust agent affordances to induce useful partial world models across related tasks? We study this by increasing the difficulty of the blocks environment with more blocks to be manipulated and arranged on the tabletop that requires longer horizon for planning exacerbating the need to guide search as the branching factor increases. We keep the $\mathcal{I}_{\texttt{agn}}$ (see Sec.~\ref{sec:formalization}) consistent from previous configuration of 3 blocks in both LLM as the affordance model alongside the full world model $P_{\texttt{LLM}}(s'|s,o)$, and the partial world model $P_{\texttt{AFF-LLM}}(s'|s,o)$. We report in Table.~\ref{tab:5-7-blocksenv_mcts} that with increasing difficulty in task configuration, our approach considerably reduces the branching factor resulting in fewer mcts simulations needed, and higher average accumulated reward during planning.

\begin{tcolorbox}[colback=gray!10, colframe=gray!10, arc=4pt, boxrule=0pt]
\textbf{Key Insight II:} Partial-models informed by distribution-robust agent affordances capture underlying invariance across tasks resulting in generalization of search efficiency (See Table~\ref{tab:5-7-blocksenv_mcts}).
\end{tcolorbox}

\section{Related Work}
\label{sec:relatedwork}
Generative models as world models constitute a broad research area. Recent works consider pre-training large-scale world models~\citep{yang2023learning}. While this is appealing, it is computationally expensive, requiring large amounts of data, compute, resulting in a non-trivial process to pre-train a world model. Moreover, in the presence of large scale knowledge representation in LMs, collecting data online from scratch is non-trivial, simply due to the challenges presented by such approach i.e. sample efficiency or dependence on visiting each transition tuple sufficient number of times (e.g. in count-based method).

\citet{bruce2024genie} proposed to learn a generative interactive model analogous to a simulator, trained on 200k hours of video data, inferring the latent actions between each frame a with the latent action model to predict the next frame. Our approach is complimentary as it can be leveraged to induce a partial model to make predictions on only \emph{affordable} latent actions, thereby reducing hallucinations of unrealistic futures as highlighted in their limitations.   

Closest to our work, \citet{hao2023reasoning}[RAP] also utilize plug-and-play with LLMs although their focus both in domains ( such as blocksworld which considers a PDDL specification of the planning problem) and research question ( such as planning efficacy with full models) is different. For example, their focus in on Math like reasoning tasks, such as GSM8k which are typically i.i.d. in nature. On the contrary, our work is focused on sequential-decision making problems such as long-horizon table top tasks with online on-policy RL. We argue that RAP can be viewed analogous to a full model, as it is limiting the action space by sampling a fixed number of actions as opposed to using generalized affordances. Our approach is complimentary to their work by leveraging distribution-robust affordances that induce partial world models. 

More recent work~\citep{benechehab2024zero} explores using LLM as a  transition model, but due to high inference cost, explores a dynamics model as a function of the states alone as opposed to state-action pairs. We note however, that they are not focused on online search and planning, but instead in an offline RL setting. Furthermore, the central focus of our work is the generalized affordance framework which extends prior work~\cite{khetarpal2021temporally, khetarpal2020can} from a single task to multi-task perspective.

\section{Discussion}
\label{sec:discussion}
This work generalized the concept of affordances to multi-task setting by delineating task-agnostic and task-specific intents. This perspective explains when and why certain affordances specific to an agent embodiment generalize across tasks, and when they might be task-specific. Our theoretical analysis established that tow findings. First, we identified that agents achieving task-agnostic intents must encode a predictive partial model of the world informed by distribution-robust affordances for a certain task distribution. Second, we showed that while near-perfect partial models enable improved search efficiency, an adaptive corrective model enjoys the potential exponential speed up of the partial model, yet being robust to missing intents from the full model. Broadly, our work suggests that while naively using LLMs as world models for planning and search can be highly inefficient, posing them as partial models through intents and affordances offers tremendous potential in guiding the search especially in a large action space. 

Future work requires overcoming the limited nature of specifying task-agnostic intents in the form of prompt engineering. One approach that is much more promising in this direction considers program synthesis for intent-completion~\cite{cherif2025cracking} identifying affordances from trajectories, resulting in coded partial models of the world that can be reused in more complex sequential task settings. Another promising direction requires the formalism of affordances to account for dynamic agent morphologies, which is scope for future work. Further, we anticipate more benefits from adapting the agent's latent representation of the pre-trained world model in light of online experiential data~\cite{silver2025welcome}. 

\section*{Impact Statement}
This work studies how affordances enable partial world models with LLMs. By explicitly modeling the limits of posing LLMs as full world models, our approach aims to find more efficient ways to search for RL agents. We presented work whose goal is to advance the field of Machine Learning. There are many potential societal consequences of our work, none which we feel must be specifically highlighted here.

\section*{Acknowledgements}

The authors are grateful to David Abel for his thorough and constructive feedback on a draft of this paper, and to Anna Koop and Benjamin Van Roy for their insightful reviews of an earlier version.

\bibliography{references}
\bibliographystyle{icml2026}

\newpage
\onecolumn
\appendix
\section*{Appendix}
\section{Additional Empirical Details}

\subsection{Details of the Approach}

\begin{figure*}[th!]
    \centering
    \includegraphics[width=0.60\linewidth]{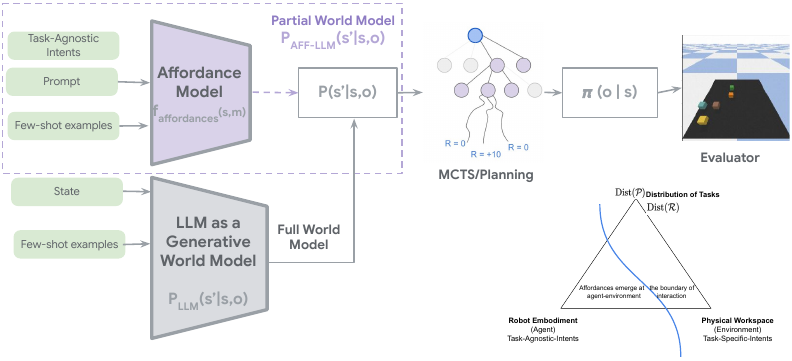}
    \caption{\textbf{Our Approach} leverages LLMs as partial world models using affordances as depicted here. Affordance-informed partial models are then used to guide MCTS in table top robotics task. We posit LLMs as-is to be generative world models, which we refer to as full world models. Partial Models are induced via affordances generated by LLMs using task-agnostic intents specified in the prompt.}
    \label{appfig:approach_overview}
\end{figure*}

\subsection{ Task-Agnostic Intents used in Table~\ref{tab:3blocksenv_mcts}, Table~\ref{tab:5-7-blocksenv_mcts}}
\label{app:task-agnostic-intents}
\begin{enumerate}
\setlength\itemsep{-0.30em}
    \item A block is movable only if that color is present in the current state and when nothing is on top of it.
    \item A block placement on x,y coordinates which has an existing block is valid if it exactly matches the coordinates of the existing block in the current state AND if the block color is present in the current state.
    \item A block placement on x,y coordinates which has an existing block is valid if it has sufficient overlap with the coordinates of the existing block in the current state AND if the block color is present in the current state.
    \item A block placement on x,y coordinates is valid if no other block is present in those coordinates.
\end{enumerate}

\section{Theoretical Derivations}
\subsection{Proof of Theorem 1.}
\label{app:prooftheorem1}

\textbf{Setup.} As we are interested in recovering an agent's partial model across a fixed distribution of tasks i.e. $T^i \sim \text{Dist}(\mathcal{T})$, we consider the tasks which share a common underlying environment.
I.e. the transition function is fixed, while the reward function across tasks (and potentially the discount factor) varies. Henceforth, we interchangeably use `goals' or `rewards' in place of `tasks'. 

We consider tasks composed of primitive sub-tasks of the following form, 
\begin{itemize}
    \item \textit{Actions}. The agent takes some desired abstract action $o \in \mathcal{O}$.
    
    \item \textit{Outcomes}. The last abstract action $O = o$ taken by the agent terminates in some desired state(s) $g\subseteq \mathcal S$, denoted $X(s \in g)$. 
    This can be thought of as the simplest temporally extended intents \ref{def:intents}

    \item \textit{Goals}. The agent steers the environment into some desired state $s\in \mathcal S$ using the actions afforded to it, over any number of time steps, denoted $F(s\in g)$.
    
\end{itemize}
Here $F$ is the Eventually operator and $X$ is the Next operator defined in linear temporal logic (LTL) \citep{pnueli1977temporal}.
Abusing notation, $T$ also denotes the LTL expression corresponding to the abstract task $T$, which evaluates to True if the agent-environment trajectory $\tau$ which achives (satisfies) the task specification (denoted $\tau \models T$) and False otherwise.

We are interested in agents that are capable of solving tasks with non-trivial complexity, i.e. consisting of multiple sub-tasks to be performed in a specified order.
We denote these sequences of sub-tasks $T_1, T_2, \ldots, T_n$ (Definition~\ref{def: depth n tasks}) using angled brackets $\langle T_1, T_2, \ldots, T_n\rangle$, to denote the LTL expression that evaluates to true if $T_1$ is satisfied, and then $T_2$, and so on. 
See \citep{richens2025general} for the construction of these LTL expressions. 

\textbf{Example.} The task `move the red block from $x$ to $y$' can be expressed as $\langle (O = o), X(s\in g)\rangle$ where $o$ is the abstract action `move block (red) from position (x) to position (y)', and $g$ is the set of environment states where the red block is in position y. 
Note the sequential nature of this task, which required not only taking the abstract action, but to achieve the desired outcome. This is to account for imperfect affordances, whereby the abstract action $O = o$ may not terminate in the desired state. 

\textbf{Example.} With the aforementioned task, we may want to relax the requirements on the agent, so that it can re-attempt the task if it fails to move the red block to the desired position. E.g. it could pick up the wrong block, or due to imperfect sensorimotor control fail to place the block in the desired location.
This may require chaining together different actions depending on outcomes (e.g. if the red block is placed in $y'$ applying the action `move the red block from (y') to (y)'). 
This is now a purely outcome-based goal, which the agent can achieve by any valid sequence of actions without a time horizon constraint, and so we can represent it as $T = \langle F(s \in g)\rangle$ where $g$ is all states where a red block is in position $y$.

We also allow for tasks with some ambiguity in how they are solved. 
For example, there can be multiple goal states whereby reaching any one of them would satisfy the overall task. 
More generally, there can be multiple sub-tasks---themselves sequential tasks---which are sufficient to achieve the overall task. 
We can construct specifications for these composite tasks by taking the Boolean OR over sequential tasks $T = T_1 \vee T_2$.

\textbf{Example.} The task of providing dinner could be achieved by (order a delivery THEN collect the delivery) OR (buy groceries THEN prepare a meal THEN serve the meal). 
This can be expressed as $ T = T_1\vee T_2$ where $T_1 = \langle (O = \text{order delivery}), X (S = \text{delivery collected})\rangle$ and $T_2 = \langle (O = \text{buy groceries}]),  X (S = \text{meal prepared})\rangle$. 

Each of these abstract actions must be afforded to the agent, e.g. an agent that has no phone or internet access will not be afforded the abstract action `order a delivery'.

The tasks we consider consist of sequences of abstract actions, with the task depth $n$ being the maximum number of sub-tasks the agent would have to complete to achieve the overall task (Definition~\ref{def: depth n tasks}). 
Intuitively, tasks with a higher depth $n$ are more complex as they require a larger number of sub-tasks to be completed, and hence require more planning.  
For example, if we consider an explicitly model-based agent, tasks with a high depth require more accurate world models compared to e.g. myopic tasks ($n=1$), due to compounding errors. 
Therefore to establish that an agent has an implicit world model, we must assume it is competent at tasks of non-trivial depth $n$ (Definition \ref{def:agent}).
This is a relaxation of the assumptions on the agent in \citep{richens2025general} as we do not assume the agent satisfies a strict regret bound for all tasks, but only on average over the task distribution, and relax the regret bound to a soft bound where the agent can violate $\zeta$ with some probability.  
As in \citep{richens2025general} we also make the simplifying assumptions that the agent follows a deterministic policy. Turning to the environment, Assumption~\ref{ass:communicatingCMDP} is a standard assumption in MDPs. 

\restatemaintheorem*

\begin{proof}
We use the approach in Sec.~\ref{sec:formalization} to return an estimate $\hat P(s'\mid o, s)$ given $\pi(o\mid h)$. Following the proof of Theorem 1 in \citep{richens2025general}, we define the following tasks.  Let $(s, o) \in \mathcal{AF}_{I^\rightarrow}$ and let $T_{o_1}(k, n)$ denote task which is a disjunction over all tasks $\bigvee_{r \leq k}\bigvee_{i \in nCr}\varphi(r, i)$ where 
\begin{equation}
    \varphi_i(r,n) := \langle {T}_0, \underbrace{{T}_1, {T}_2,\ldots {T}_1, {T}_2'}_{n \text{ times}}\rangle 
\end{equation}
where the agent
\begin{itemize}
    \item[i)] takes action $O = o_1$ in the first time step ($ T_0 = (O = o_1)$), then transitions eventually to $S = s$ and takes action $O= a$, (${T}_1 = F((S = s, O = o_1))$),
    \item[ii)] transitions next to a goal state which is either $S = s'$, (${T}_2 = X (S = s')$) or $S\neq s'$, (${T}_2' = X(S \neq s')$),
    \item[iii)] returns eventually to $S = s$ and takes action $O = o_1$, (${T}_1$), and repeats the cycle ii)-iii) a total of $n$ times, with the transition ${T}_2 = X(S'= s)$ occurring $r$ times and the transition ${T}'_2 = X(S \neq s')$ occurring $n-r$ times. 
\end{itemize}
I.e. ${T}_{o_1}(k, n)$ requires the agent to take abstract action $O = o_1$ in state $S = s$ a total of $n$ times, and at most $k$ of these actions should terminate in the state $S = s'$.  
Note that $n$ here is the number of times the agent takes $O = o$ in $S = s$, but the goal depth is $2n + 1$.

Consider the task ${T}_{o_2}(k, n)$ that is identical to ${T}_{o_1}(k, n)$ except that to achieve the goal the agent must i) initially take action $O = o_2$ instead of $O = o_1$ and ii) must terminate in $S = s$ at least $k+1$ times. 

Note we can construct the task $ T_{o_1,o_2} (k, n) = {T}_{o_1}(k, n)\vee {T}_{o_2}(k, n)$ for any pair of abstract action $o_1, o_2 \in \mathcal{AF}_{I^\rightarrow}(\tilde s)$ such that $o_1 \neq o_2$ (where by assumption at least two such actions exist).
As $ {T}_{o_1}(k, n)$ and $ {T}_{o_2}(k, n)$ are mutually exclusive we have, 
\begin{equation}
    P(\tau \models  T_{o_1,o_2} (k, n) \mid \pi, s_0) = P(\tau \models  {T}_{o_1}(k, n) \mid \pi, s_0) + P(\tau \models  {T}_{o_2}(k, n) \mid \pi, s_0)
\end{equation}
and for any $\pi$ only one of the terms on the right hand side is non-zero. 
Hence to determine the optimal policy we can evaluate $\max_\pi P(\tau \models {T}_{o_1}(k, n) \mid \pi, s_0)$ and $\max_\pi P(\tau \models {T}_{o_1}(k, n) \mid \pi, s_0)$. First we consider the case where $\delta = 0$ (the agent satisfies a strict regret bound for all $ T_i \sim \text{Dist}(\mathcal T)$).
As $\pi$ is deterministic by assumption, let $O = o$ be the agent's action at $t = 0$. 
If the agent chooses $O = o_1$ then only ${T}_{o_1}(k, n)$ can be satisfied and likewise for $O = o_2$, and any other action choice achieves the task with probability 0. 

${T}_{o_1}(k, n) = \bigvee_{r\leq k}\bigvee_{i\in nCr} \varphi_{i}(r, n)$ and all $\varphi_{i}(r, n), \varphi_{j}(r', n)$ are mutually exclusive, i.e. for $r\neq r'$ or $i \neq j$, $\tau \models \varphi_{i}(r, n) \implies \tau \not\models \varphi_{j}(r', n)$ and vice versa, and hence, 
\begin{equation}
    P(\tau \models {T}_{o_1}(k, n) \mid \pi, s_0) = \sum\limits_r \sum\limits_{i=1}^{nCr} P(\tau \models \varphi_{i}(r, n)\mid \pi, s_0)
\end{equation}
and by Lemma 6 in \citep{richens2025general} the maximum probability that this goal can be satisfied is given by,
\begin{equation}
    \max\limits_\pi P(\tau \models  {T}_{o_1}(k, n) \mid \pi, s_0) = \sum\limits_{r = 0}^k P_n(X = r) = P_n(X \leq k)
\end{equation}
where
\begin{equation}
    P_n(X = r) := \frac{n!}{(n-r)!r!}P_{ss'}(a)^r (1-P_{ss'}(a))^{n-r}
\end{equation}
is the binomial probability mass function and $P_n(X \leq k)$ is the cumulative distribution function. Likewise if the agent chooses the initial action $o_2$ (pursuing ${T}_{o_2}(k, n)$), it can achieve this task with maximum probability,
\begin{equation}
    \max\limits_\pi P(\tau \models {T}_{o_2}(k, n) \mid \pi, s_0) = \sum\limits_{r = k+1}^n P_n(X = k) = P_n(X > k)
\end{equation}

If the initial action is neither $o_1$ nor $o_2$ then the agent achieves task  $T_{o_1,o_2} (k, n)$ with probability zero.
Therefore, for an agent satisfying a $\zeta$ regret bound the optimal agent's policy $\pi(o \mid s_0 ;  T_{o_1,o_2} (k, n))$ witnesses the following inequalities, depending on its choice of abstract action $o_1$ or $o_2$,
\begin{align}
    O = o_1 &\implies \,P_n(X \leq k) \geq P_n(X > k)-\zeta\label{eq: ineq 1}\\
    O = o_2 &\implies \, P_n(X > k) \geq P_n(X \leq k)-\zeta\label{eq: ineq 2}\\
    O \not\in \{o_1, o_2\}&\implies \max\{P_n(X \leq k), P_n(X > k)\} \leq \zeta\label{eq: useless inequality}
\end{align}

For $\zeta < 1/2$ we can ignore \eqref{eq: useless inequality} as $ P_n(X > k) = 1-P_n(X \leq k)$ so one of these terms must be larger than $\zeta$. 

For ease of notation we denote $P_{ss'}(a) = p$. 
The median of the binomial distribution $X = m$ is an integer $0 \leq m \leq n$ that satisfies $np-1 \leq m \leq np + 1$. 
For $\zeta = 0$ we can exactly determine the median by varying $k$ and finding the unique value for which $k^*$ such that $P_n(X> k^*-1)\geq P_n(X\leq k^*-1)$ and $P_n(X\leq k^*)\geq P_n(X> k^*)$, yielding $k^* = m$. 

For $\zeta > 0$ we can no longer determine the median with certainty, but can bound its value. 
Using $P(X > k) = 1 - P(X \leq k)$ and $P(X > k-1) = P(X \geq k)$, and \eqref{eq: ineq 1} and \eqref{eq: ineq 2} yields, 

\begin{align}
     O = o_1 &\implies \,P_n(X \leq k) \geq \frac{1-\zeta}{2}\\
    O = o_2 &\implies \, P_n(X \geq k) \geq \frac{1-\zeta}{2}\\
\end{align}

Let $k = m^-$ be the largest value of $k$ for which $O = o1$, and $k = m^+$ be the smallest value of $k$ such that $O = o_2$. The true median $m$ must lie between these values. These yield the tightest bounds, 

\begin{align}
     O = o_1 &\implies \,P_n(X \leq m^-) \geq \frac{1-\zeta}{2}\label{eq: ineq 1 new}\\
    O = o_2 &\implies \, P_n(X \geq m^+) \geq \frac{1-\zeta}{2}\label{eq: ineq 2 new}
\end{align}

Applying the upper tail Chebyshev inequality $P_n(X \geq \mu + t \sigma) \leq 1/(1+t^2)$ with $m^+ = \mu + t \sigma$ to \eqref{eq: ineq 2 new} and using $\mu = np$, $\sigma^2 = n p (1-p)$ yields, 
\begin{equation}\label{eq: upper bound m}
    m^+ - \mu \leq \sigma \sqrt{\frac{1 + \zeta}{1 - \zeta}}
\end{equation}
and likewise for the lower tail Chebyshev bound applied to \eqref{eq: ineq 1 new} we recover, 

\begin{equation}
     \mu - m^- \leq \sigma \sqrt{\frac{1 + \zeta}{1 - \zeta}}
\end{equation}

For $\delta = 0$ these bounds would allow us to recover bounds on $\mu$ (and hence $p$) as done in \citep{richens2025general}, as the lowest value of $k$ such that the agent's policy returns $O = o_2$ for task $T_{o_1,o_2} (k, n)$ must lie between $m^-$ and $m^+$, yiedling in a bounded error estimate for the median. However, for $\delta > 0$ by Definition~\ref{def:agent} for any given task $T_i$ there is a probability of at most $\delta$ that the agent's policy violates the $\zeta$ regret bound, and so the agent may return $O = o_2$ for $k < m^-$ or $O = o_1$ for $k > m^+$, albeit with probability at most $\delta$. 

Let $f(k, n)\in \{0, 1\}$ $\forall$ $k\in 1, \ldots, n$ be an indicator variable for the action returned by the agent's policy for task $T_{o_1,o_2} (k, n)$, where for a given $k, n$ we have $O = o_1 \rightarrow f(k, n) = 0$ and $O = o_2 \rightarrow f(k, n) = 1$.  
To estimate the median we fit a step function $F(\hat k)$ which predicts $F(k) = 0$ $\forall$ $k \leq \hat k$ else $F(k) = 1$. 
The error of this model at reproducing $f(k)$ is, 

\begin{equation}\label{eq: error equation}
    E(\hat k, n) = \sum_{k=1}^{\hat k}f(k) + \sum_{k = \hat k + 1}^n(1-f(k))
\end{equation}

We will take $\hat k := \argmin \text{error}(\hat k, n)$ to be our estimate of the median, and derive an upper bound on the expected error $|\hat k - m|$. 
Our aim is to determine a probabilistic lower bound on the value of $\hat k$.
Hence, we can consider the case where we set $f(k \geq m^-) = 1$, as this can only reduce the value of $\hat k$. 
For all $k< m^-$, we can only have $f(k) = 1$ if the agent's policy violates the $\zeta$ regret bound, which by Definition~\ref{def:agent} occurs for each such $k$ independently with a probability at most $\delta$. 
Therefore we can set $P(f(k) = 1) = \delta$ for all $k< m^-$, as this is maximizing the probability of $f(k) = 1$ for $k < m^-$ which reduces the value of $\hat k$.
In this case the $\hat k$ is given by \eqref{eq: error equation} but replacing $n$ with $m^-$. 

Let $y(k) = 2 f(k) - 1$, yielding $\mathbb E[Y(k)] = 1 - 2 \delta$. 
Let $E(0)$ be the error for the model that predicts $f(k) = 1$ for all $k$. 
We can express the error $E(k) = E(k, m^-)$ as a telescoping sum, 

\begin{align}
    E(k) &= E(0) + \sum\limits_{i = 1}^{m^-}(E(k) - E(k-1))\\
    &= E(0) +  \sum\limits_{i = 1}^{m^-} Y_i
\end{align}

Let $S(k) = \sum_{i=1}^{m_-} Y_i$ be a random walk with drift $\mu = E[Y_i] = 1 - 2 \delta$. 
As $E(0)$ is a constant offset, $\hat k= \argmin_{0 \leq k \leq m^-} E(k) = \argmin_{0 \leq k \leq m^-} S(k)$, and so minimizing $E(k)$ is equivalent to minimizing $S(k)$.
We can therefore derive a worst case probabilistic lower bound on $\hat k$ by deriving a similar bound for the minimum of the random walk $S_k$.   
We desire a bound on the event $A = \{\hat k \leq m^- - \Delta \}$ where $\Delta > 0$. 
$A$ implies a weaker condition $B = \{\min_{0 \leq k \leq m^- - \Delta} S_k \leq S_{m^-}\}$, therefore if we bound the probability of $B$ we upper bound the probability of $A$. 
Using $P(B) = P(\min_{0 \leq k \leq m^- - \Delta} [S_k \leq S_{m^-}]) = P(\max_{0 \leq k \leq m^- - \Delta}[ (S_{m^-} - S_k) \geq 0])$. 
For $S_{m^-} - S_k = \sum_{i= k + 1}^{m^-}Y_i$ and $Y_i$ are iid, this is equal to a sum over $m^- - k$ iid Bernoulli variables, and the number of terms in the sum ranges from $m^-$ (for $k = 0$) to $\Delta$. 
We can therefore replace this sum with $W_l := \sum_{j=1}^l Y'_j$ with $l$ ranging from $\Delta$ to $m^-$, yielding $P(B) = P(\max_{\Delta \leq l \leq m^-}[W_l \geq 0]) = P(\bigcup_{l = \Delta}^{m^-}[W_l \geq 0])$.
Applying the union bound gives $P(B) \leq \sum_{l = \Delta}^{m^-}P(W_l \geq 0)$.

Next, we can upper bound $P(W_l \geq 0)$ using the Chernoff bound to give,
\begin{equation}
    P(W_l \geq 0) \leq \inf\limits_{s>0} E[e^{s W_l}]= \inf\limits_s E[e^{s \sum\limits_{i=1}^l Y_i}] = \inf\limits_s E[e^{s Y}]^l
\end{equation}

Using $E[e^{s Y}] = (1-\delta) e^{-s} + \delta e^s$ and minimizing w.r.t $s$ gives, 

\begin{equation}
    P(W_l \geq 0) \leq 
    \left(2\sqrt{\delta (1-\delta)}\right)^l
\end{equation}

for $\delta < 1/2$.
Finally, we can determine a probabilistic upper bound, \begin{equation}
    P(\hat k \leq m^- - \Delta) \leq \sum\limits_{l = \Delta}^{m^-}P(W_l \geq 0) \leq \sum\limits_{l = \Delta}^{m^-} \left(2\sqrt{\delta (1-\delta)}\right)^l 
\end{equation}

Letting $\rho := 2\sqrt{\delta (1-\delta)}$ the sum over this geometric series gives, 

\begin{equation}
    P(\hat k \leq m^- - \Delta) \leq \rho^\Delta \left(\frac{1 - \rho^{m^- - \Delta}}{1- \rho}  \right)\leq \frac{\rho^\Delta}{1 - \rho}
\end{equation}

We can repeat these steps for deriving an upper bound to $\hat k$ as a deviation from $m^+$, and (by symmetry) arrive at

\begin{equation}
    P(\hat k \geq m^+ + \Delta) \leq \frac{\rho^\Delta}{1 - \rho}
\end{equation}

Finally, using the (identical) upper bounds for $m^\pm$ given by \eqref{eq: upper bound m} we recover, 

\begin{equation}
    P(|\hat k - n p| \geq \sigma \sqrt{\frac{1 + \zeta}{1- \zeta}} + \Delta) \leq \frac{\rho^\Delta}{1 - \rho}
\end{equation}

Finally, using $\hat p = \hat k / n$ as our estimate for $p$, and $\sigma = \sqrt{n p (1-p)}$ yields a probabilistic error bounds, 

\begin{equation}
    P(|\hat p -  p| \geq  \phi + \epsilon) \leq \frac{\rho^{n\, \epsilon}}{1 - \rho}
\end{equation}

where
\begin{equation}
    \phi = \sqrt{\frac{p(1-p)(1 + \zeta)}{n(1- \zeta)}}
\end{equation}

and where we have defined $\epsilon := \Delta / n$, and $\epsilon \in [0, 1]$. 
Finally, we relax the bound using $p(1-p) \leq 1/4$ to give a simpler expression. 
\end{proof}

\subsection{Adaptive Corrected Model and Proof of \cref{theorem:fastersearch}.}
\label{app:prooftheorem2}

Let us consider a single task ${T}_i$.  We use a Las Vegas algorithm (Monte Carlo with restarts until success) over sequences of intents of length $L$. If the terminating reward is received before or when all $L$ intents are fully performed, the task terminates with success, otherwise the sequence has failed and another sequence is sampled until success.

In a sampled sequence (a trial), each intent is sampled uniformly and independently of the past. That is, the probability of any intent $I_j$ for the full model is $\Pfull(I_j) = 1/n$. For the partial model it is $\Ppartial(I_j) = 1/k$ if $I_j\in \Ipartial$, 0 otherwise. Hence, each sequence has a probability $n^{-L}$ to be sampled for the full model, and $k^{-L}$ for the partial model.

For any single trial, the success probability for the full model is 
\begin{align*}
    \Pfull&(\success_i) \\
    &= \sum_{I_1\dots I_L \in \Ifull^L} P(\success_i\mid I_1\dots I_L) \Pfull(I_1\dots I_L) \\
    &= \sum_{I_1\dots I_L \in \Ifull^L} P(\success_i\mid I_1\dots I_L) \Pfull(I_1)\times\dots\times \Pfull(I_L) \\
    &= n^{-L}\sum_{I_1\dots I_L \in \Ifull^L} P(\success_i\mid I_1\dots I_L)
\end{align*}
and it is well known that the expected number $\Nfull_i$ of trials before success on task ${T}_i$ is
\begin{align*}
    \Nfull_i = \frac1{\Pfull(\success_i)} = \frac{n^L}{\sum_{I_1\dots I_L \in \Ifull^L} P(\success_i\mid I_1\dots I_L)}
\end{align*}
The expected number of trials necessary to solve every task is
\begin{align*}
    \Nfull = \sum_{i\leq |\mathcal{T}|} \Nfull_i = n^L \sum_{i\leq |\mathcal{T}|}  \frac{1}{\sum_{I_1\dots I_L \in \Ifull^L} P(\success_i\mid I_1\dots I_L)}
\end{align*}
We assume that every task can be solved in finite time by the full model.

Similarly, for the partial model the expected number $\Npartial$ of trials to solve every task is 
\begin{align*}
    \Npartial = k^L \sum_{i\leq |\mathcal{T}|}  \frac{1}{\sum_{I_1\dots I_L \in \Ipartial^L} P(\success_i\mid I_1\dots I_L)}
\end{align*}
Since in general $k \ll n$, if the intents of the partial model have a good enough coverage to solve all tasks, the potential speedup of $(n/k)^L$ compared to the full model can be gigantic.
But if the partial model misses \emph{any necessary} intent to solve \emph{any} task, $\Npartial$ is infinite.
Even with all the \emph{necessary} intents to ensure that each task can be solved, it may still miss \emph{important} intents
that could help to solve tasks much faster.
To avoid this, we define a \emph{corrected} partial model which assigns a small probability to use missing intents from the full model:
For each intent $I_j$,
\begin{align*}
    \Pcor(I_j\mid \epscor) &= (1-\epscor) \frac1k\indicator{I_j \in \Ipartial} + \epscor\frac1{n-k}\indicator{I_j \notin \Ipartial}
\end{align*}
Observe that setting $\epscor=0$ recovers the partial model, while $\epscor=(n-k)/n$ recovers the full model.

Therefore,
for a fixed $\epscor$,
and for a given sequence of intents $J = I_1\dots I_L$, the probability of success of the corrected model for task ${T}_i$,
where $m_J \leq L$ of the intents come exclusively from the full model ($m_J$ is the number of `mistakes'), 
is
\begin{align*}
    \Pcor(\success_i | \epscor) 
    &=  \sum_{J = I_1\dots I_L \in \Ifull^L} P(\success_i\mid J)\prod_{I_j\in J}\Pcor(I_j\mid \epscor) \\
    &= \sum_{J \in \Ifull^L} 
    P(\success_i\mid J)\left(\frac{1-\epscor}{k}\right)^{L-m_J}\left(\frac{\epscor}{n-k}\right)^{m_J}\,.
\end{align*}
But how do we set $\epscor$?
We can at least define the optimal $\epscor^{*}_{i}$ for each task $i$:
\begin{align*}
    \epscor^{*}_{i} = \argmax_{\epscor\in[0,1]} \Pcor(\success_i\mid \epscor)\,,
\end{align*}
but in general $\epscor^{*}_{i}$ is not known.
So we construct an \emph{adaptive} corrected model that is competitive within a factor $(2L+3)e$ with the optimal $\epscor^{*}_{i}$ for each task $i$:
Before sampling a sequence of $L$ intents, we sample $m \in \{0,\dots, 2L+2)\}$ with uniform probability $P(m) = 1/(2L+3)$,
and we set $\epscor = m/(2L+2)$.
The probability of an intent (given $m$) for this adaptive model is
\begin{align*}
    \Pada(I_j\mid m) = \Pcor\left(I_j\mid \epscor = \frac{m}{2L+2}\right)\,.
\end{align*}

For a given sequence of intents $J = I_1\dots I_L$, the probability of success of the adaptive model for task ${T}_i$ is
\begin{align}\label{eq:Pada}
    \Pada(\success_i) 
    &=  \sum_{J = I_1\dots I_L \in \Ifull^L} \sum_{m= 0}^{2L+2} \frac1{2L+3} P(\success_i\mid J)\prod_{I_j\in J}\Pcor\left(I_j\mid \epscor = \frac{m}{2L+2}\right) \\
    &= 
    \frac1{2L+3}
    \sum_{J \in \Ifull^L} 
    \sum_{m= 0}^{2L+2}
    P(\success_i\mid J)\left(\frac{1-\epscor}{k}\right)^{L-m_J}\left(\frac{\epscor}{n-k}\right)^{m_J}\,, \quad\quad \epscor= \frac{m}{2L+2}\,.
\end{align}
\begin{lemma}\label{lem:Pada}
For each task $i\leq |\mathcal{T}|$,
there exists an $m\in\{0..2L+2\}$ such that for every sequence $J\in\Ifull^L$,
\begin{align*}
  \prod_{I_j\in J}\Pcor\left(I_j\mid \epscor = \epscor^*_i\right) \leq e\times \prod_{I_j\in J}\Pcor\left(I_j\mid \epscor = \frac{m}{2L+2}\right)\,.
\end{align*}
\end{lemma}
\begin{proof}
Rename $\eps^{*}  = \epscor^{*}_i$ and $\eps = m/(2L+2)$.
Then
\begin{align*}
    \sum_{I_j\in J}\ln \frac
    {\Pcor\left(I_j\mid \epscor = \epscor^{*}_i\right)} 
    {\Pcor\left(I_j\mid \epscor = \frac{m}{2L+2}\right)}
    =
    (L-m_J)\ln\frac{1-\eps^*}{1-\eps} + m_J\ln\frac{\eps^*}{\eps}\,.
\end{align*}
\textbf{Case $\eps^* \leq 1/2$.}
Choose $m = \min \{m\in\{0..2L+2\}: \frac{m}{2L+2}\geq \eps^*\}$, that is $m = \lceil(2L+2)\eps^*\rceil$.
Then, first, $\eps = m / (2L+2) \geq \eps^*$ by definition of $m$
and thus $\eps^*/\eps \leq 1$.
Second,
\begin{align*}
    \eps = \frac{\lceil(2L+2)\eps^*\rceil}{2L+2}\leq \eps^* + \frac1{2L+2}\,,
\end{align*}
and thus,
\begin{align*}
    \frac{1-\eps^*}{1-\eps} \leq \frac{1-\eps^*}{1-\eps^* - 1/(2L+2)}
    = 1+ \frac{1/(2L+2)}{1-\eps^*-1/(2L+2)} = 1+ \frac1{(1-\eps^*)(2L+2) - 1} 
    \leq 1+\frac1{L}\,,
\end{align*}
where we used $\eps^* \leq 1/2$ on the last inequality.
Therefore, for every sequence $J\in\Ifull^L$, using $\ln(1+x)\leq x$,
\begin{align*}
    \sum_{I_j\in J}\ln \frac
    {\Pcor\left(I_j\mid \epscor = \epscor^{*}_{i}\right)} 
    {\Pcor\left(I_j\mid \epscor = \frac{m}{2L+2}\right)}
    \leq (L-m_J) \ln \left(1+\frac1{L}\right) + m_J \ln 1 \leq 1\,.
\end{align*}
Exponentiating finishes the proof for this case.

\textbf{Case $\eps^* > 1/2$.}
Similar as Case $\eps^* \leq 1/2$ but choosing instead $m = \max\{m\in\{0..2L+2\}: 1-\frac{m}{2L+2} \geq 1-\eps^* \}$.
\end{proof}

\restateefficiencytheorem*
\begin{proof}
Applying \cref{lem:Pada} to \cref{eq:Pada}, and dropping all $m \in\{0..2L+2\}$ from the sum except the one satisfying \cref{lem:Pada} we obtain that
\begin{align*}
    \Pada(\success_i) 
    &\geq
    \frac1{e(2L+3)}
    \sum_{J \in \Ifull^L} 
    P(\success_i\mid J)\left(\frac{1-\epscor^{*}_i}{k}\right)^{L-m_J}\left(\frac{\epscor^{*}_i}{n-k}\right)^{m_J} \\
    &= \frac1{e(2L+3)}
    \max_{\epscor \in[0,1]}
    \sum_{J \in \Ifull^L} 
    P(\success_i\mid J)\left(\frac{1-\epscor}{k}\right)^{L-m_J}\left(\frac{\epscor}{n-k}\right)^{m_J}\\
    &= \frac1{e(2L+3)}
    \max_{\epscor \in[0,1]}\Pcor(\success_i\mid \epscor)\,,
\end{align*}
where the last line is by definition of $\epscor^{*}_{i}$.
The result follows by considering the expected time to success for task $i$ and summing over all tasks.
\end{proof}

\end{document}